\def\year{2019}\relax
\documentclass[letterpaper]{article} 
\usepackage{aaai19}  
\usepackage{times}  
\usepackage{helvet}  
\usepackage{courier}  
\usepackage{url}  
\usepackage{graphicx}  
\usepackage{algorithmic}
\usepackage{algorithm}
\usepackage{url}
\usepackage{amsmath}
\usepackage{color}
\usepackage{subfig}
\usepackage{booktabs}
\usepackage{balance} 
\usepackage{multirow}
\usepackage{amsmath,amssymb} 
\usepackage{amsthm}
\usepackage{amsmath}
\newtheorem{theorem}{Theorem}
\newtheorem{lemma}{Lemma}
\usepackage{bm}
\usepackage{rotating}
\begin{document}
%
\title{SADIH: Semantic-Aware DIscrete Hashing}
\author{Zheng Zhang$^1$, Guo-sen Xie$^2$, Yang Li$^1$, Sheng Li$^3$, Zi Huang$^1$\\
$^1$ The University of Queensland, Australia\\
$^2$ Inception Institute of Artificial Intelligence, UAE\\
$^3$ University of Georgia, USA\\
\{darrenzz219, gsxiehm\}@gmail.com; y.li9@uq.net.au; sheng.li@uga.edu; huang@itee.uq.edu.au
}
\maketitle
\begin{abstract}
Due to its low storage cost and fast query speed, hashing has been recognized to accomplish similarity search in large-scale multimedia retrieval applications. Particularly, supervised hashing has recently received considerable research attention by leveraging the label information to preserve the pairwise similarities of data points in the Hamming space. However, there still remain two crucial bottlenecks: 1) the learning process of the \textit{full pairwise similarity preservation} is computationally unaffordable and unscalable to deal with big data; 2) the available category information of data are not well-explored to learn discriminative hash functions. To overcome these challenges, we propose a unified Semantic-Aware DIscrete Hashing (SADIH) framework, which aims to directly embed the transformed semantic information into the asymmetric similarity approximation and discriminative hashing function learning. Specifically, a semantic-aware latent embedding is introduced to asymmetrically preserve the full pairwise similarities while skillfully handle the cumbersome $n\times n$ pairwise similarity matrix. Meanwhile, a semantic-aware autoencoder is developed to jointly preserve the data structures in the discriminative latent semantic space and perform data reconstruction. Moreover, an efficient alternating optimization algorithm is proposed to solve the resulting discrete optimization problem. Extensive experimental results on multiple large-scale datasets demonstrate that our SADIH can clearly outperform the state-of-the-art baselines with the additional benefit of lower computational costs.
\end{abstract}
\section{Introduction}
In the big data era, recent years have witnessed the ever-growing volume of multimedia data with high dimensionality. This is made possible by the emergence of large-scale similarity measurement technique with high computational efficiency \cite{zhang2017marginal,zhang2018binary}. Different from the traditional indexing technique \cite{Lew2006}, hashing yields a scalable similarity search mechanism with acceptable accuracies in the fast Hamming space \cite{JDWang2018}. Technically, hashing generally compresses the high-dimensional data instances into compact binary codes (typically $\leq$ 128-dim), in which the similarity and structural information are preserved from the original data. In this paper, we will mainly focus on the learning-based hashing methods that are formulated by the data-dependent hash encoding strategy, which has shown better retrieval performance than data-independent (or learning-free) hashing schemes, such as locality-sensitive hashing (LSH) \cite{LSH} and its variants \cite{BrianKSH,jiang2015revisiting}.

A common problem of learning-based hashing methods is to construct similarity-preserving hash functions, which generate similar binary codes for nearby data items. Many such hashing learning methods have been proposed to enable efficient similarity search and can be broadly grouped into two categories: unsupervised and supervised hashing.

Unsupervised methods typically encode samples as binary codes by exploring data distribution without label or relevances \cite{zhang2018binary,zhang2018highly}. They learn hash codes/functions based on the semantic gap principle \cite{smeulders2000content}, \textit{i.e.,} the difference in structures formed within the high- and low-level descriptors. Representative unsupervised hashing methods include manifold learning based hashing and quantization based hashing. Manifold learning based hashing tries to discover the neighborhood relationship of data points in the learned binary codes, such as spectral hashing (SH) \cite{SH}, multiple feature hashing (MFH) \cite{MFH}, scalable graph hashing (SGH) \cite{SGH} and ordinal constraint hashing (OCH) \cite{liu2018ordinal}. Quantization based hashing aims to achieve the minimal quantization error, such as iterative quantization (ITQ) \cite{ITQ} and quantization-based hashing \cite{song2018quantization}. Due to the absence of semantic label information, unsupervised hashing is usually inferior to supervised hashing, which can produce state-of-the-art retrieval results.

Supervised hashing generates discriminative and compact hash codes/functions by leveraging the supervised semantic information from data such as pairwise similarity or relevant feedback. Many supervised hashing methods have been proposed to enable efficient similarity search. Representative methods in this group include semi-supervised hashing (SSH) \cite{SSH}, latent factor hashing (LFH) \cite{LFH}, fast supervised hashing (FastH) \cite{FastH}, column sampling based discrete supervised hashing (COSDISH) \cite{COSDISH}, etc. It has been studied to construct encoding functions in a designed kernel space, such as binary reconstructive embedding (BRE) \cite{BRE}, kernel-based supervised hashing (KSH) \cite{KSH}, the kernel variant of ITQ \cite{ITQ}, supervised discrete hashing (SDH) \cite{SDH}, SDH with relaxation (SDHR) \cite{SDHR} and fast SDH (FSDH) \cite{FSDH}. The kernel based hashing algorithms have been demonstrated to achieve promising performance.

To further improve the retrieval performance, many deep hashing models \cite{erin2015deep,shen2018unsupervised,li2016feature,lai2015simultaneous,lin2015deep} have been introduced over the past few years, where the nonlinear feature embeddings learned by deep neural networks were typically shown to achieve higher performance than hand-crafted descriptors. As we know, semantic hashing \cite{salakhutdinov2009semantic} is the pioneering work of using deep machine for hashing. However, these deep hashing models are complicated and need pre-training, which is inefficient in real applications. Moreover, there might be concern about the encoding time of the training and test data.

Although achieving progress, current supervised similarity-preserving hashing methods are still facing severe challenges. First, to \textit{avoid} using the full $n$ $\times$ $n$ pairwise similarity matrix, these methods employ sampling strategies in the training phase to reduce the large computation and memory overhead. In such a case, they fail to capture the full structures residing on the entire data, which inevitably results in information loss and unsatisfactory performance. Their objectives would be suboptimal for realistic search tasks, and such methods become inappropriate for large-scale retrieval tasks. Second, only preserving the pairwise similarities transformed from labels clearly excludes the category information of data from the training step. In this way, these methods can not transfer the discriminative information from labels into the learned binary codes, resulting in inferior performance. Third, since the discrete optimization introduced by binary constraint leads to an NP-hard mixed integer programming problem, most of them usually solve it by relaxing the binary variables into continuous ones, followed by thresholding or quantization. However, such relaxation strategy can amplify the quantization errors, which may greatly influence the quality of the learned binary codes and degrade the performance.

To address the aforementioned problems, we propose a novel discriminative binary code learning framework, dubbed Semantic-Aware DIscrete Hashing (SADIH), for fast scalable supervised hashing. Specifically, we introduce an asymmetric similarity-preserving strategy that can preserve the discrete constraint and reduce the accumulated quantization error between binary code matrix and the well-designed latent semantic-aware embedding. During the learning step, such trick can skillfully handle the huge $n\times n$ pairwise similarity matrix, and preserve the discriminative category information into the learned binary codes. Meanwhile, we also develop a novel semantic-aware encoder-decoder paradigm to guarantee the high-quality latent embedding. In particular, an encoder projects the visual features of an image into a latent semantic space, and in turn consider the latent semantic-aware embedding as an input to a decoder which reconstructs the original visual representation. As such, our learning framework not only can effectively preserve the discriminative semantic information into the learned binary codes and hashing functions, but efficiently approximates the full pairwise similarities without information loss. Furthermore, an alternating algorithm is developed to solve the resulting problem, where each subproblem can be optimized efficiently, yielding satisfactory solutions.To sum up, the main contributions of this work are:

(1) A novel semantic-aware discrete hashing framework is proposed to simultaneously consider the full pairwise similarities ($n\times n$) and the category information into the joint learning objective. SADIH aims to generate discriminative binary codes which can successfully capture the entire pairwise similarities as well as the intrinsic correlations between visual features and semantics from different categories.

(2) We introduce a latent semantic embedding space which can reconcile the structural difference between the visual and semantic spaces, meanwhile preserve the discriminative structures in the learned binary codes.

(3) An asymmetric similarity approximation loss is developed to reduce the accumulated quantization error between the learned binary codes and the latent semantic-aware embeddings. Meanwhile, a supervised semantic-aware autoencoder is constructed to jointly perform the data structural preservation and data reconstruction. The well-designed alternating optimization algorithm with guaranteed convergence is applied to produce the high-quality hash codes.

\section{Semantic-Aware Discrete Hashing}
\subsection{Basic Formulation}
This work mainly focuses on supervised hashing to enable efficient semantic similarity search by Hamming ranking of compact hash codes. Suppose we have $n$ $d$-dimensional data points, denoted as $\bm X = [\bm x_1, \cdots, \bm x_n] \in \Re^{d\times n}$, and their associated semantic labels are $\bm Y = [\bm y_1, \cdots, \bm y_n]\in \{0,1\}^{c\times n}$, where $c$ is the number of classes. The $i$-th column of matrix $\bm Y$, \textit{i.e.} $\bm y_i = [0,\cdots 1,\cdots,0]^T \in \Re^c$, is the label vector of the $i$-th sample, and $y_{ji} = 1$ indicates $\bm x_i$ belongs to the $j$-th class. Notably, supervised hashing also contains a pairwise similarity matrix $\bm S \in \{-1,1\}^{n\times n}$ obtained from semantic correlations such as labels used in this paper. Specifically, $s_{ij} = 1$ means that data items $i$ and $j$ are semantically similar and share at least one label, while $s_{ij} = -1$ indicates items $i$ and $j$ are semantically dissimilar.

The goal of supervised hashing aims to learn $l$ hashing functions to project the data $X$ into a discriminative Hamming space, and generate a binary code matrix $\bm B = [\bm b_1, \cdots, \bm b_n] \in \{-1,1\}^{l\times n}$. Moreover, the learned binary codes should preserve the semantic similarities indicated in $\bm S$. The commonly-used objective function \cite{KSH} quantizes the approximation error between the Hamming affinity and semantic similarity matrix using
\begin{align}\label{eq_1}
\min_{\bm B} \| r\bm S - \bm B^T \bm B\|_F^2~s.t.~\bm B \in \{-1,1\}^{l\times n},
\end{align}
where $\|\cdot\|_F$ is the Frobenius norm. In this model, the inner product of any two binary codes reflects the opposite of the Hamming distance, and can be used to approximate the corresponding similarity labels. Due to its effectiveness, this model has become a standard formulation for supervised hashing learning \cite{LFH,FastH,COSDISH}. However, there are still several deficiencies. \textbf{First}, such a symmetric binary affinity form is limited in matching the real-valued ground truth. Importantly, the optimization on symmetric discrete constraint is time-consuming, which makes it hard to adapt for large-scale datasets \cite{neyshabur2013power}. \textbf{Second}, owing to its computation and memory prohibition, the full similarity matrix $\bm S$ is usually avoided using in the training step. An alternative strategy is to sample a small subset for training, which inevitably causes information loss and suboptimal results. \textbf{Third}, directly transforming labels into pairwise similarities loses the category information of training data, which can not preserve the discriminative characteristics into the learned binary codes. \textbf{Finally}, most methods solve the discrete optimization problem by relaxing the discrete constraint by omitting the sign function. However, the approximate solution is obviously suboptimal and often generates low-quality hashing codes.

Therefore, we present an efficient \textbf{Semantic-Aware DIscrete Hashing (SADIH)} framework to alleviate the above limitations. In the method, the asymmetric hamming affinity approximation, latent semantic embedding and encoder-decoder paradigm are simultaneously considered to guarantee discriminative binary codes and hashing functions.

\subsection{Objective Function}
\subsubsection{Asymmetric Similarity Approximation Loss} To fully explore the entire similarities on $n$ available points, we introduce a simple but effective semantic-aware constraint, \textit{i.e.,} $\mathcal{\bm V} = \bm W^\top\bm Y$ where $\bm W \in \Re^{c\times l}$, to asymmetrically approximate the ground-truth affinity. Meanwhile, the label information are embedded into the latent semantic embedding $\mathcal{\bm V}$. Particularly, the matrix $\bm W^\top$ can be viewed as the category-level basis matrix of the latent semantic features $\mathcal{\bm V}$, because $\bm v_i = \bm W^\top\bm y_i$, where each item in $\bm v_i$ contains the category partition factor. Moreover, using the real-valued embeddings can produce more accurate approximation of similarity, and reduce the accumulated quantization error \cite{dong2008asymmetric,luo2018scalable}. Based on the asymmetric hashing learning \cite{shrivastava2014asymmetric}, we replace one of the binary codes $\bm B$ in (\ref{eq_1}), and consider its robust model
\begin{align}\label{eq_2}
&\min_{\bm B, \bm W} \| l\bm S - \mathcal{\bm V}^T \bm B\|_{21}\nonumber\\&~s.t.~\mathcal{\bm V} = \bm W^T\bm Y, \bm B \in \{-1,1\}^{l\times n},
\end{align}
which $\|\bm A\|_{21} = \sum_{i=1}^n\|\bm a^i\|_2$ denotes the $l_{21}$-norm of matrix $\bm A$ and $\bm a^i$ is the $i$-th row of matrix $\bm A$. The $l_{21}$-norm is robust to noise or outliers based on the rotation-invariance property. It is noteworthy that this simple constraint can enable the model to effectively exploit all the $n$ data points for training (shown in optimization) without any sampling. Moreover, it also can more precisely measure the quantization between the given similarities $\bm S$ and the learned asymmetric affinity.

\subsubsection{Semantic-Aware Autoencoder} The discriminative binary codes for training data can be learned based on model (\ref{eq_2}), but there still remain two concerns. On one hand, the latent semantic embedding only leverages the label information, while the inherent characteristics embedded in the training data are not well-explored to capture the instance-level features. On the other hand, (\ref{eq_2}) can not be generalized to the out-of-sample cases for efficient query generation. To this end, we formulate the linear semantic-aware autoencoder scheme, which optimizes against the following objective:
\begin{align}\label{eq_3}
&\min_{\{\bm P_i, \bm c_i\}_{i=1}^2, \mathcal{\bm V}, \bm W} \| \bm X - (\bm P_2(\bm P_1 \bm X + \bm c_1 \bm 1^T) + \bm c_2 \bm 1^T) \|_F^2 \nonumber\\&+ \gamma \mathcal R(\bm P_2, \bm P_1)     ~s.t.~\mathcal{\bm V} = \bm W^T\bm Y, \mathcal{\bm V} = \bm P_1 \bm X + \bm c_1 \bm 1^T,
\end{align}
where $\bm P_1 \in \Re^{l\times d}$ and $\bm P_2 \in \Re^{d\times l}$ are the encoding and decoding matrices, $\mathcal R(\cdot) = \|\cdot\|_F^2$ is the regularization term to avoid overfitting, $\bm c_1$ and $\bm c_2$ are the biased vectors, and $\gamma$ is a weighting parameter. It is clear that \textit{this model can make use of semantic attributes in $\mathcal{\bm V}$ as an intermediate level clue to associate low-level visual features with high-level semantic information}. However, when we project the visual $d$-dim features into the lower $l$-dim (typically $l\ll d$) semantic space, this may encounter the imbalanced projection problem, \textit{i.e.} the variances of the projected dimensions vary severely \cite{SSH}. To this end, inspired by ITQ \cite{ITQ}, we may change the coordinates of the whole feature space through an adjustment rotation. For eliminating the bias variables, we reformulate the above problem into a relaxed optimization with an orthogonal transformation:
\begin{align}\label{eq_4}
&\min_{\bm P_1,\bm P_2, \mathcal{\bm V}, \bm W} \| \bm X - \bm P_2 \mathcal{\bm V}\|_F^2 + \beta \|\mathcal{\bm V} - \bm P_1 \bm X\|_F^2 + \gamma \mathcal R(\bm P_2)
\nonumber\\& ~s.t.~\mathcal{\bm V} = \bm W^T\bm Y, \bm P_1^T \bm P_1 = \bm I.
\end{align}
From Eqn. (\ref{eq_4}), we can see that the preferred latent attributes satisfy $\mathcal{\bm V} = \bm P_1 \bm X$ with minimum reconstruction error $\bm X = \bm P_2 \mathcal{\bm V}$. Given the orthogonal transformation, the overall variance can be effectively diffused into all the learned dimensions through the adjustment rotation, and the underlying characteristics hidden in data $\bm X$ are uncovered and transferred into the semantic embeddings. Importantly, the encoder can be used as a linear hash function for new queries.
\subsubsection{Joint Objective Function} To preserve the interconnection between the semantic-aware similarity approximation and preferable latent semantic space construction, SADIH combines the asymmetric similarity approximation loss given in (\ref{eq_2}) and semantic-aware autoencoder scheme given in (\ref{eq_4}) into one unified learning framework. Such a learning framework can minimize the intractable full pairwise similarity-preserving error, meanwhile interactively enhances the qualities of the learned binary codes and discriminative latent semantic-aware embeddings:
\begin{align}\label{eq_6}
&\min_{\bm B, \bm W, \bm P_1,\bm P_2, \mathcal{\bm V}} \| l\bm S - \mathcal{\bm V}^T \bm B\|_{21} +\alpha \|\bm X - \bm P_2 \mathcal{\bm V}\|_F^2 \nonumber\\ &~~~~~~~~~~~~~~~~~~~~~~~~~~ + \beta \|\mathcal{\bm V} - \bm P_1 \bm X\|_F^2 + \gamma \mathcal R(\bm P_2)
\nonumber\\& ~s.t.~\bm B \in \{-1,1\}^{l\times n},\mathcal{\bm V} = \bm W^T\bm Y, \bm P_1^T \bm P_1 = \bm I,
\end{align}
where $\alpha$ is a weighting parameter. Furthermore, it is clear that the first term makes sure the asymmetric correlations between the discriminative binary codes $\bm B$ and latent semantic representations $\mathcal{\bm V}$. Therefore, the encoding matrix $\bm P_1$ can capture the discriminative information embedded in the latent semantic space, and the hashing codes for out-of-sample $\bm x_{t}$ can be directly generated by $\bm b = sgn(\bm P_1\bm x_{t})$, where $sgn(\cdot)$ denotes the element-wise sign function.

\section{Optimization}
The key of our optimization is to learn discriminative binary codes $\bm B$ and find a proper latent-embedding space $\mathcal{\bm V}$ with the data structural preservation $\bm P_1$ and data reconstruction $\bm P_2$. However, problem (\ref{eq_6}) is non-convex to all variables, and involves a discrete constraint, which leads to an NP-hard problem. Thus, we propose an alternating optimization algorithm to obtain a local minima.

We first rewrite problem (\ref{eq_6}) as an equivalent form:
\begin{align}\label{eq_7}
&\min_{\bm B, \bm W, \bm P_1,\bm P_2} \| l\bm S \!-\! \left(\bm W^T\bm Y\right)^T \bm B\|_{21}\! +\!\alpha \|\bm X \!-\! \bm P_2 \bm W^T\bm Y\|_F^2 \nonumber\\ &~~~~~~~~~~~~~~~~~~+ \beta \|\bm W^T\bm Y - \bm P_1 \bm X\|_F^2 + \gamma \mathcal R(\bm P_2, \bm W^T\bm Y)
\nonumber\\& ~s.t.~\bm B \in \{-1,1\}^{l\times n}, \bm P_1^T \bm P_1 = \bm I.
\end{align}
Next, we iteratively update each variable with an alternative manner, \textit{i.e.}, updating one when fixing others.

\textbf{$\bm B$-Step}: When fixing $\bm W$, $\bm P_1$ and $\bm P_2$, we can update $\bm B$ by using two discrete learning strategies. The first learning scheme optimizes $\bm B$ with a constant-approximation solution, inspired by \cite{COSDISH}. Specifically, we push the $l_{21}$-norm loss to a more strict $l_1$-norm loss, \textit{i.e.,}
\begin{align}\label{eq_8}
&\min_{\bm B} \| l\bm S \!-\! \left(\bm W^T\bm Y\right)^T \bm B\|_1~s.t.~\bm B \in \{-1,1\}^{l\times n}.
\end{align}
where $\|\cdot\|_1$ denotes the $l_{1}$-norm. Problem (\ref{eq_8}) reaches its minimum when $\bm B = sgn(\bm W^T\bm{Q})$, where $\bm Q = l\bm{YS}$ can be calculated beforehand. We denote this as \textbf{\textit{SADIH-L1}}.

Alternatively, we can directly optimize the $l_{21}$-norm loss with an equivalent transformation. We first define $\bm D \in \Re^{n\times n}$ as a diagonal matrix, the $i$-th diagonal element of which is defined as $\bm d_{ii}$ = $1/ 2\|\bm u^i\|$, where $\bm u^i$ is the $i$-th row of $\left(l\bm S-\bm Y^T\bm W\bm B\right)$. In this way, we need to solve
\begin{align}\label{eq_9}
&\min_{\bm B} Tr \left(\big(l\bm S - \bm Y^T\bm W \bm B\big)^T \bm D \big(l\bm S - \bm Y^T\bm W \bm B\big)\right) \nonumber\\
&~s.t.~\bm B \in \{-1,1\}^{l\times n},
\end{align}
which can be equivalently rewritten as
\begin{align}\label{eq_10}
&\min_{\bm B}Tr\left( \bm R^T\bm B \bm D \bm B^T \bm R \right)  -2 tr \left(\bm B^T \bm M\right)\nonumber\\
&~s.t.~\bm B \in \{-1,1\}^{l\times n},
\end{align}
where $\bm R = \bm W^T\bm Y$, $\bm M = \bm W^T\bm D\bm Q$, and Tr($\cdot$) is the trace norm. For this binary quadratic program problem, we employ the discrete cyclic coordinate descent (DCC) method \cite{SDH} to sequentially learn each row of $\bm B$ while fixing other rows. Let $\bm b^\top$ be the $k$-th row of $\bm B$, $k = 1, \cdots, l$, and $\bm\bar{\bm B}$ be the matrix of $\bm B$ excluding $\bm b$. Similarly, let $\bm r^T$ and $\bm q^T$ be the $k$-th row of $\bm R$ and $\bm M$, respectively. $\bm \bar{\bm R}$ and $\bm \bar{\bm M}$ are the matrix of $\bm R$ excluding $\bm r$ and the matrix of $\bm M$ excluding $\bm q$, respectively. Then, we have
\begin{align}\label{eq_11}
&Tr\left( \bm R^\top\bm B \bm D \bm B^\top \bm R \right)\nonumber\\
&= Tr\left(\left(\bm\bar{\bm B}^\top\bm\bar{\bm R}+\bm b\bm r^\top\right) \bm D \left(\bm\bar{\bm R}^\top\bm\bar{\bm B}+\bm r\bm b^\top\right)\right)\\
&= 2Tr\left(\bm b\bm r^\top \bm D\bm\bar{\bm R}^\top \bm\bar{\bm B}\right)+Tr\left(\bm b\bm r^\top \bm D \bm r\bm b^\top\right) + const.\nonumber
\end{align}
Since $Tr(\bm b\bm r^\top \bm D \bm r\bm b^\top) = Tr(\bm r^\top \bm D \bm r\bm b^\top\bm b) = n\sum_i \bm d_{ii} \bm r^\top\bm r$, this term is a constant \textit{w.r.t.} $\bm b$.

Similarly, $tr \left(\bm B^T \bm M\right) = \left[\begin{array}{cc} \bm b^\top \bm q & \bm b^\top\bm\bar{\bm M}\\ \bm\bar{\bm B}^\top\bm q& \bm\bar{\bm B}^\top\bm\bar{\bm M}\end{array}\right]$, and then
\begin{align}\label{eq_12}
tr \left(\bm B^T \bm M\right) = \bm q^\top\bm b+const.
\end{align}

Therefore, Eqn. (\ref{eq_10}) can be reformulated as
\begin{align}\label{eq_13}
&\min_{\bm b} Tr \left( \left(\bm r^\top \bm D\bm\bar{\bm R}^\top \bm\bar{\bm B} - \bm q^\top\right) \bm b\right)\nonumber\\&~s.t.~\bm b \in \{-1,1\}^{l},
\end{align}
which has a closed-form solution:
\begin{align}\label{eq_14}
\bm b = sgn \left(\bm q - \bm\bar{\bm B}^\top\bm\bar{\bm R}\bm D\bm r\right).
\end{align}
We can see that each bit $\bm b$ is calculated based on the pre-learned ($l$-$1$) bits $\bm\bar{\bm B}$. We iteratively update each bit until it converges to a set of optimal codes $\bm B$.


\textbf{$\bm W$-Step}: When fixing $\bm B$, $\bm P_1$ and $\bm P_2$, the objective function (\ref{eq_7}) \textit{w.r.t.} $\bm W$ is degenerated to
\begin{align}\label{eq_15}
&\min_{\bm W} Tr \left(\big(l\bm S - \bm Y^T\bm W \bm B\big)^T \bm D \big(l\bm S - \bm Y^T\bm W \bm B\big)\right) \!\\ &+\!\alpha \|\bm X \!-\! \bm P_2 \bm W^T\bm Y\|_F^2+ \beta \|\bm W^T\bm Y - \bm P_1 \bm X\|_F^2 + \gamma \| \bm W^T\bm Y\|_F^2.\nonumber
\end{align}
The closed-form solution of $\bm W$ is given by setting the derivation of (\ref{eq_15}) to zero, \textit{i.e.,}
\begin{align}\label{eq_16}
\bm W = \bm L(\bm{QD}&\bm{B}^\top +\bm{YX}^\top (\alpha \bm P_2+\beta \bm P_1))\nonumber\\&~~(\bm{BDB}^\top+\alpha \bm P_2\bm P_2^T + (\beta +\gamma)\bm I)^{-1},
\end{align}
where $\bm L = (\bm{YY}^\top)^{-1}$ can be calculated beforehand.

\begin{table*}[!t]\vspace{-.2cm}
\begin{center}
\caption{The averaged retrieval result comparison (MAP score and precision of top 100 samples) and computation time efficiency (in seconds) on \textbf{CIFAR-10} using $512$-dimensional GIST features, and both \textbf{SUN397} and \textbf{ImageNet} datasets using $4096$-dim deep CNN features from VGG19 $fc7$. The best performances have been displayed in boldface.} \label{Table_1}
\resizebox{1\textwidth}{!}{
\begin{tabular}{|c||ccccccc||cccccc||cccccc|}
\hline
\multirow{2}{*}{\textbf{Methods}}&& \multicolumn{18}{c|}{\textbf{CIFAR-10}}\\
\cline{2-20}
&&\multicolumn{6}{c||}{\textbf{MAP Score}}&
\multicolumn{6}{c||}{\textbf{Precision@top100}}&
\multicolumn{6}{c|}{\textbf{Computation Time (in second)}}\\
\hline
\textbf{\# bits} && 8-bits & 16-bits & 32-bits & 64-bits & 96-bits & 128-bits 
  & 8-bits & 16-bits & 32-bits & 64-bits & 96-bits & 128-bits 
  & 8-bits & 16-bits & 32-bits & 64-bits & 96-bits & 128-bits \\
\hline
ITQ&&0.1528 &0.1552 &0.1642 &0.1700 &0.1719 &0.1757 
	&0.2012 &0.2341 &0.2798 &0.3035 &0.3146 &0.3295
    & 0.98 & 1.40 & 2.49 & 4.35 & 7.01 & 10.07 \\
SGH&&0.1411 &0.1522 &0.1643 &0.1711 &0.1759 &0.1781 
	&0.1995 &0.2487 &0.2866 &0.3085 &0.3280 &0.3329 
    &4.47 & 6.27 &  8.30 &  14.62 &   18.71 &  25.16\\
CBE&&0.1134 &0.1168 &0.1335 &0.1454 &0.1552 &0.1564
	&0.1370 &0.1545 &0.1927 &0.2449 &0.2706 &0.2791 
    &29.80 & 29.92 & 29.90 & 29.30 & 29.53  & 29.53\\
DSH&&0.1441 &0.1437 &0.1548 &0.1589 &0.1622 &0.1651 
	&0.1649 &0.1986 &0.2358 &0.2498 &0.2680 &0.2738
    &0.35 & 0.39 &  0.49 & 0.81 &  1.03 & 1.36\\
MFH&&0.1428 &0.1483 &0.1543 &0.1605 &0.1649 &0.1638
	&0.1908 &0.2410 &0.2645 &0.2944 &0.3146 &0.3124 
    &25.33 & 26.44 & 35.30 & 37.94 & 49.19  & 55.42\\
OCH&&0.0744 &0.1395 &0.2299 &0.3048 &0.3459 &0.3720 
	&0.1530 &0.3276 &0.4895 &0.5994 &0.6489 &0.6765 
    & 9.29 & 9.25 & 8.97 & 13.94 & 14.53 & 9.76\\
ITQ-CCA&&0.2736 &0.3157 &0.3361 &0.3519 &0.3588 &0.3552 
	&0.3815 &0.4205 &0.4417 &0.4574 &0.4650 &0.4634
    & 1.17 & 2.14 & 3.36 & 6.88 & 11.25 & 16.62\\
KSH&&0.2477 &0.2786 &0.3044 &0.3202 &0.3330 &0.3349
	&0.2511 &0.2849 &0.3068 &0.3164 &0.3283 &0.3298
    & 38.25 & 77.04 & 171.54 & 354.14 & 501.57 & 729.75\\
LFH&&0.2845 &0.3849 &0.5013 &0.5891 &0.6144 &0.6109 
	&0.2163 &0.2889 &0.4126 &0.5039 &0.5353 &0.5279 
    & 4.13 & 7.22 & 12.07 & 18.25 & 29.53 & 40.03\\
COSDISH&&0.4990 &0.5573 &0.6263 &0.6155 &0.6399 &0.6551
	&0.4347 &0.4818 &0.5483 &0.5284 &0.5591 &0.5728 
    & 5.65 & 14.78 & 34.84 & 128.13 & 302.06 & 541.51\\
SDH&&0.2855 &0.3956 &0.4478 &0.4610 &0.4735 &0.4789 
	&0.3633 &0.4920 &0.5298 &0.5419 &0.5527 &0.5543 
    & 29.40 & 70.30 & 142.73 & 278.70 & 537.74 & 874.45\\
SDHR&&0.2547 &0.3782 &0.4422 &0.4583 &0.4727 &0.4782 
	&0.3221 &0.4783 &0.5306 &0.5478 &0.5503 &0.5561
     & 37.08 & 41.69 & 54.94 & 115.12 & 240.87 & 278.52\\
FSDH&&0.3211 &0.4196 &0.4285 &0.4629 &0.4703 &0.4708
	&0.4170 &0.5046 &0.5231 &0.5399 &0.5433 &0.5449 
    & 7.58 & 11.45 & 11.75 & 8.28 & 14.00 & 15.85\\
FastH&&0.2965 &0.3679 &0.4358 &0.4715 &0.4917 &0.5018
	&0.3582 &0.4414 &0.5107 &0.5528 &0.5751 &0.5875 
    & 95.51 & 219.51 & 302.46 & 596.94 & 1013.51 & 834.36\\
\hline
\textbf{SADIH}&&0.4647 &0.6992 &0.7263 &0.7271 &\textbf{0.7367} &0.7319
	&0.4204 &0.6404 &\textbf{0.6682} &0.6609 &\textbf{0.6703} &0.6666 
    & 1.37 & 4.16 & 19.88 & 65.30 & 158.48 & 315.40  \\
\textbf{SADIH-L1}&&\textbf{0.6131} &\textbf{0.6994} &\textbf{0.7277} &\textbf{0.7298} &0.7289 &\textbf{0.7348}
	&\textbf{0.5487} &\textbf{0.6472} &0.6647 &\textbf{0.6673} &0.6621 &\textbf{0.6697}  
    & 0.36 & 0.41 & 0.47 & 0.70 & 0.79 & 0.97 \\
\hline\hline
\multirow{2}{*}{\textbf{Methods}}&& \multicolumn{18}{c|}{\textbf{SUN397}}\\
\cline{2-20}
&&\multicolumn{6}{c||}{\textbf{MAP Score}}&
\multicolumn{6}{c||}{\textbf{Precision@top100}}&
\multicolumn{6}{c|}{\textbf{Computation Time (in second)}}\\
\hline
\textbf{\# bits} && 8-bits & 16-bits & 32-bits & 64-bits & 96-bits & 128-bits 
  & 8-bits & 16-bits & 32-bits & 64-bits & 96-bits & 128-bits 
  & 8-bits & 16-bits & 32-bits & 64-bits & 96-bits & 128-bits \\
\hline
\hline
ITQ&&0.0138 &0.0449 &0.1139 &0.2304 &0.3166 &0.3803
	&0.0177 &0.0579 &0.1199 &0.1955 &0.2377 &0.2658
    & 6.26 & 7.04 & 9.34 &13.65 & 18.32  & 23.94\\
SGH&&0.0149 &0.0516 &0.1229 &0.2379 &0.3164 &0.3830 
	&0.0197 &0.0726 &0.1408 &0.2159 &0.2619 &0.2903 
    & 15.23 & 17.46  & 21.83 & 30.41 & 38.61 &  46.81\\
DSH&&0.0069 &0.0216 &0.0480 &0.0878 &0.1231 &0.1612 
	&0.0090 &0.0358 &0.0667 &0.1081 &0.1368 &0.1615 
    & 1.43 & 1.62  & 2.16 & 3.14 & 4.33 &  5.51\\
CBE&&0.0035 &0.0094 &0.0266 &0.0814 &0.1480 &0.2054 
	&0.0058 &0.0183 &0.0468 &0.1063 &0.1597 &0.1971 
    & 116.35 &112.35  & 113.56 & 113.35 & 112.76 & 112.84\\
MFH&&0.0136 &0.0464 &0.1135 &0.2255 &0.3116 &0.3804
	&0.0199 &0.0664 &0.1317 &0.2093 &0.2556 &0.2874
    & 181.95 & 178.09 & 175.22 & 202.37  & 227.33 & 246.25\\
OCH&&0.0519 &0.1770 &0.2302 &0.3291 &0.3869 &0.3655 
	&0.0418 &0.1481 &0.1978 &0.2960 &0.3529 &0.3233 
    &58.72 & 50.17 & 55.71 & 57.07 & 57.58 & 60.73 \\
ITQ-CCA&&0.1426 &0.2130 &0.2725 &0.3958 &0.4201 &0.4177 
	&0.0659 &0.1833 &0.3092 &0.4926 &0.5498 &0.5567 
    & 14.00 & 14.96 & 21.44 & 37.70 & 28.45 & 63.39 \\
KSH&&\textbf{0.3400} &0.4072 &0.4361 &0.4276 &0.4497 &0.4578 
	&0.0605 &0.0717 &0.0759 &0.0745 &0.0776 &0.0785 
    & 58.75 & 203.49 & 400.48 & 770.64 &  1159.23 & 1747.01 \\
LFH&&0.0530 &0.1536 &0.2814 &0.3533 &0.3674 &0.4013 
	&0.0454 &0.1278 &0.2553 &0.3183 &0.3312 &0.3600 
    & 58.93 &  54.11 & 68.85 & 92.22 & 140.28 &  143.53  \\
COSDISH&&0.2135 &0.3341 &0.5143 &0.6832 &0.7160 &0.7365 
	&0.0796 &0.2141 &0.4081 &0.6096 &0.6527 &0.6755
    & 33.93 & 57.69 & 92.10 & 258.65 & 556.82 &1026.28\\
SDH&&0.1211 &0.2379 &0.3415 &0.4225 &0.4585 &0.4867
	&0.1024 &0.3531 &0.5070 &0.5990 &0.6395 &0.6672 
     & 188.58 & 191.78 & 450.55 & 1014.60 & 1839.20 & 1958.06\\
SDHR&&0.1109 &0.2249 &0.3140 &0.4152 &0.4688 &0.4842
	&0.0797 &0.3212 &0.4821 &0.6049 &0.6489 &0.6614 
    & 933.43 & 958.68 & 1150.00 & 1692.07 & 1736.54 & 1863.96\\
FSDH&&0.1064 &0.2112 &0.3365 &0.4245 &0.4604 &0.4775 
	&0.0615 &0.3224 &0.5020 &0.6046 &0.6451 &0.6557 
    & 91.21 & 92.37 & 91.34  & 97.81 & 105.39 & 116.13\\
FastH&&0.1061 &0.2027 &0.3070 &0.3865 &0.4195 &0.4386
	&0.0890 &0.2929 &0.4676 &0.5670 &0.5935 &0.6203 
    & 103.02 & 198.64 & 448.35 & 886.73 & 785.03 & 1667.14\\
\hline
\textbf{SADIH}&&0.1588 &\textbf{0.5599} &0.5909 &0.6798 &0.7332 &0.7655 
	&0.0972 &\textbf{0.5411} &0.5578 &0.6411 &0.6911 &0.7156
    &12.68 & 20.78 & 52.64 & 186.77 & 398.34 & 695.19  \\
\textbf{SADIH-L1}&&0.2938 &0.5491 &\textbf{0.6597} &\textbf{0.7576} &\textbf{0.7762} &\textbf{0.8077}
	&\textbf{0.1467} &0.4472 &\textbf{0.5956} &\textbf{0.7050} &\textbf{0.7250} &\textbf{0.7572}
    & 9.72 & 10.21 & 10.49 & 11.18 & 11.62 & 11.79  \\
\hline\hline
\multirow{2}{*}{\textbf{Methods}}&& \multicolumn{18}{c|}{\textbf{ImageNet}}\\
\cline{2-20}
&&\multicolumn{6}{c||}{\textbf{MAP Score}}&
\multicolumn{6}{c||}{\textbf{Precision@top100}}&
\multicolumn{6}{c|}{\textbf{Computation Time (in second)}}\\
\hline
\textbf{\# bits} && 8-bits & 16-bits & 32-bits & 64-bits & 96-bits & 128-bits 
  & 8-bits & 16-bits & 32-bits & 64-bits & 96-bits & 128-bits 
  & 8-bits & 16-bits & 32-bits & 64-bits & 96-bits & 128-bits \\
\hline
\hline
ITQ&&0.0442 &0.1379 &0.2626 &0.4194 &0.4991 &0.5531
	&0.0543 &0.1665 &0.2816 &0.3817 &0.4190 &0.4475
    & 38.73 & 40.45 &  42.87 & 50.23 & 59.28 &  68.59\\
SGH&&0.0386 &0.1246 &0.2544 &0.4158 &0.4952 &0.5569 
	&0.0459 &0.1509 &0.2493 &0.3381 &0.3728 &0.3957
    & 38.28 & 45.24 & 49.80 & 57.90 & 67.90 &  86.43\\
DSH&&0.0206 &0.0504 &0.1149 &0.1854 &0.2886 &0.3295 
	&0.0265 &0.0755 &0.1485 &0.2062 &0.2735 &0.2966 
    &4.65 &  5.83& 6.91 & 9.59 & 13.15 & 18.02\\
CBE&&0.0240 &0.0311 &0.0554 &0.0821 &0.1078 &0.1573 
	&0.0346 &0.0744 &0.1672 &0.2441 &0.3127 &0.4098 
    &287.12 & 254.74 &  251.37 & 243.71 & 245.01 & 270.34\\
MFH&&0.0340 &0.1118 &0.2368 &0.3937 &0.4908 &0.5570 
	&0.0446 &0.1347 &0.2376 &0.3274 &0.3698 &0.3958
    & 616.86 & 750.02 & 799.93 & 808.44 & 882.59 & 900.17 \\
OCH&&0.1302 &0.2720 &0.4187 &0.5248 &0.5705 &0.5969 
	&0.2263 &0.3728 &0.4711 &0.5665 &0.6976 &0.7141 
    & 33.27 & 35.07 & 36.89 & 36.12 & 36.96 & 37.01 \\
ITQ-CCA&&0.1616 &0.2735 &0.4090 &0.6017 &0.7278 &0.7441
	&0.0959 &0.2627 &0.4428 &0.6518 &0.7671 &0.7773
    & 66.18 & 65.77 & 65.63 & 75.71 & 81.99 & 110.89\\
KSH&&0.1549 &0.2902 &0.4115 &0.4870 &0.5141 &0.5282 
	&0.1478 &0.2257 &0.2757 &0.3079 &0.3201 &0.3251 
    & 136.77 & 237.76 & 453.38 & 834.08 & 1328.50 & 1565.16 \\
LFH&&0.2512 &0.4380 &0.5576 &0.6347 &0.4972 &0.6604
	&0.1638 &0.5186 &0.6519 &0.7138 &0.6044 &0.7337 
    & 100.45 & 105.73 & 128.35 & 184.69 &  219.43 & 168.22  \\
COSDISH&&0.2392 &0.5298 &0.7304 &0.7937 &0.8038 &0.8092
	&0.0996 &0.4663 &0.6943 &0.7640 &0.7670 &0.7740 
    & 39.85 & 61.05 & 106.87 & 425.75 & 848.97 & 1271.77  \\
SDH&&0.2684 &0.4631 &0.5975 &0.6488 &0.3597 &0.7040 
	&0.1662 &0.5377 &0.6752 &0.7235 &0.4397 &0.7550 
    & 182.21 & 187.85 & 210.91 & 614.88 & 936.21 & 1490.15\\
SDHR&&0.2703 &0.4637 &0.5843 &0.6749 &0.6653 &0.6827 
	&0.1664 &0.5417 &0.6694 &0.7373 &0.7404 &0.7494
    & 249.32 & 258.10 & 281.94 & 705.89 & 1138.48 & 1645.44\\
FSDH&&0.2818 &0.4650 &0.5987 &0.6768 &0.7063 &0.7187
	&0.1879 &0.5414 &0.6775 &0.7382 &0.7582 &0.7631
    & 136.65 & 149.12 & 137.81 & 188.14 & 192.19 & 196.87\\
FastH&&0.1613 &0.2983 &0.4145 &0.4803 &0.5119 &0.5306 
	&0.1502 &0.2251 &0.2776 &0.3045 &0.3176 &0.3257
    & 116.31 & 191.69 & 327.66 & 798.27 & 991.33 & 1369.75 \\
\hline
\textbf{SADIH}&&0.2587 &0.5790 &0.7303 &0.7964 &\textbf{0.8112} &0.8082 
	&0.0817 &0.5172 &0.7014 &0.7660 &\textbf{0.7770} &0.7710 
    & 7.70 & 17.21 & 53.38 & 188.56 & 410.69 & 732.02  \\
\textbf{SADIH-L1}&&\textbf{0.5169} &\textbf{0.7057} &\textbf{0.7767} &\textbf{0.8037} &0.8100 &\textbf{0.8165}
	&\textbf{0.4254} &\textbf{0.6554} &\textbf{0.7360} &\textbf{0.7664} &0.7745 &\textbf{0.7848}
    & 5.75 & 5.69 & 6.07 & 6.02 & 6.63 & 8.03  \\
\hline
\end{tabular}}
\end{center}\vspace{-.05cm}
\end{table*}

\textbf{$\bm F$-Step}: When fixing $\bm B$, $\bm W$ and $\bm P_2$, problem (\ref{eq_7}) \textit{w.r.t.} $\bm P_1$ becomes
\begin{align}\label{eq_17}
\min_{\bm P_1} \|\bm W^T\bm Y - \bm P_1 \bm X\|_F^2 ~s.t.~\bm P_1^T \bm P_1 = \bm I,
\end{align}
which can be solved by the following lemma.
\begin{lemma}
$\bm P_1 = \bm{UV}^T$ is the optimal solution to the problem in Eqn. (\ref{eq_17}), where $\bm U$ and $\bm V$ are the left and right singular matrices of the
compact Singular Value Decomposition (SVD) on ($\bm{XY}^\top \bm W$).
\end{lemma}
\begin{proof}
Due to page limitation, we move the proof to supplementary material.
\end{proof}

\textbf{$\bm P$-Step}: Similarly, when fixing other variables, problem (\ref{eq_7}) \textit{w.r.t.} $\bm P_2$ can be re-written as:
\begin{align}\label{eq_18}
\min_{\bm P_2} \alpha \|\bm X \!-\! \bm P_2 \bm W^T\bm Y\|_F^2 + \gamma \|\bm P_2\|_F^2.
\end{align}
The minimal $\bm P_2$ can be obtained by setting the partial derivative of Eqn. (\ref{eq_18}) to zero, and we have
\begin{align}\label{eq_19}
\bm P_2 = \bm{(\alpha \bm {XX}^\top+\gamma I)^{-1}XY}^T\bm W,
\end{align}
where $T = (\alpha \bm {XX}^\top+\gamma I)^{-1}$ can be computed beforehand.

\begin{figure*}[!t]\vspace{-.1cm}
\centering\resizebox{.995\textwidth}{!}{
\includegraphics[]{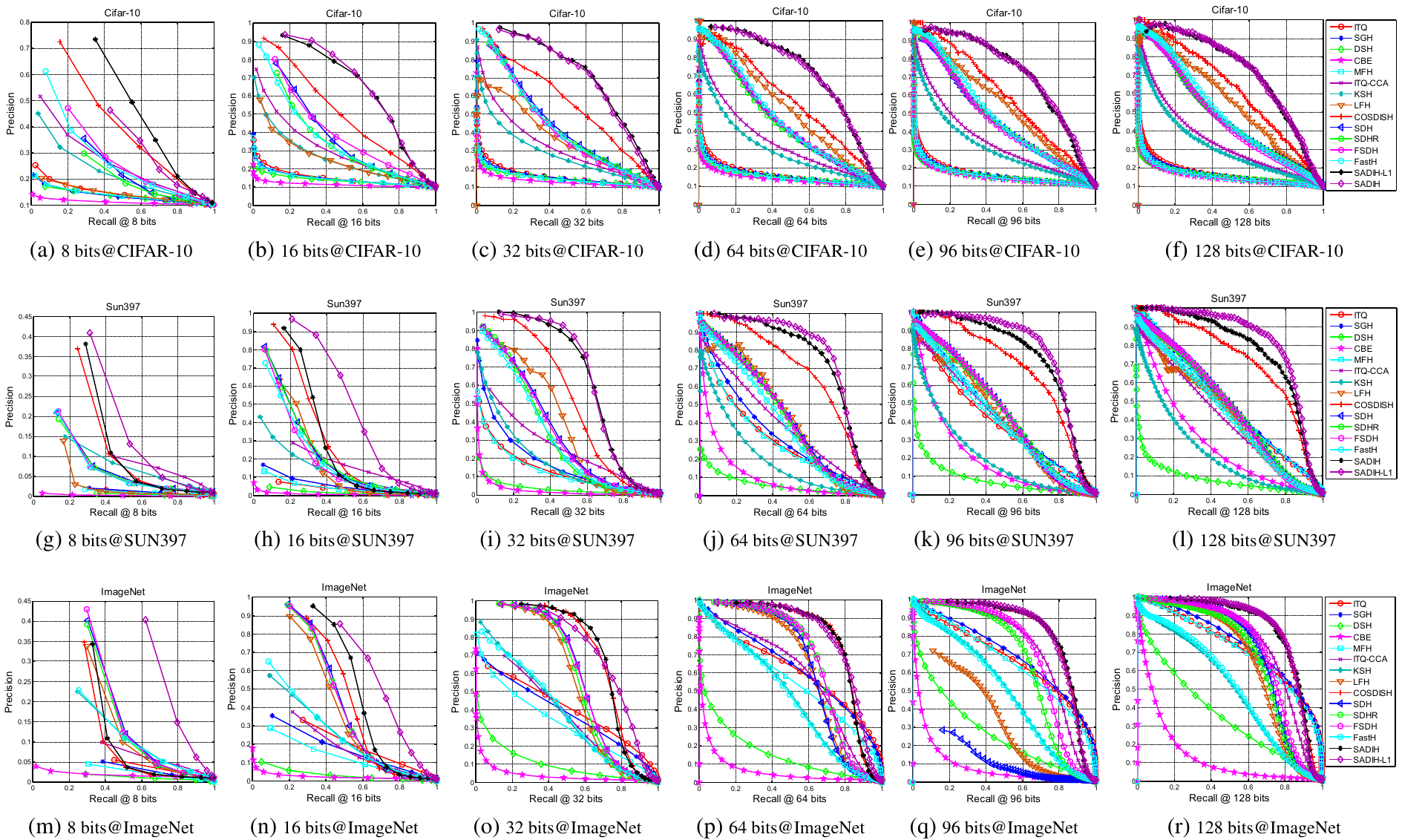}}
\caption{\small Precision-Recall curves of different methods with different code lengths on CIFAR-10, SUN397 and ImageNet.}\label{fig_2}
\end{figure*}

The proposed optimization iteratively updates four variables until satisfying the convergence criteria. The convergence of the proposed optimization algorithm is guaranteed by the following theorem.
\begin{theorem}
The alternating optimization steps of our method will monotonously decrease the value of the objective function until it converges to a local optima.
\end{theorem}
In experiments, we found that our algorithm usually can efficiently converge within $t$ = $5-8$ iterations. The main computational complexity of our algorithm comes from calculating $\bm B$ of each iteration in $\mathcal{O}(lcn)$ for SADIH and $\mathcal{O}(nc)$ for SADIH-L1. Additionally, during optimization, the maximum complexity of other steps is $\mathcal{O}(ndc)$ due to the property of matrix product and inversion, which is very efficient in practice. In general, the computational complexities of the proposed optimization algorithm on SADIH learning is linear to the number of samples $\mathcal{O}(n)$.

\section{Experiments}
In this section, we conduct extensive evaluations of our discriminative hashing methods on real datasets. All the experiments are implemented on Matlab 2013a on a standard Window PC with 64 GB RAM.

\textbf{Datasets:} We evaluate the proposed SADIH and SADIH-L1 on three publicly available benchmark databases: \textbf{CIFAR-10} \cite{Cifar10}, \textbf{Sun397} \cite{Sun397}, and \textbf{ImageNet} \cite{ImageNet}. \textit{Cifar-10} is a labeled subset of 80M tiny images, containing $60,000$ tiny color images in $10$ classes of objects with $6,000$ images per class. Each image with unique semantic label is represented by $512$-dim GIST features on the $32 \times 32$ pixels. Moreover, we employ the nonlinear anchor features to improve the interpretation of features as in \cite{SDH}. \textit{SUN-397} includes around $108k$ images pictured under $397$ different scene conditions. Each category has at least $100$ images. Each image is represented by $4,096$-dim deep cnn features from VGG19 $fc7$ \cite{VGG}. \textit{ImageNet} from ILSVRC 2012 has more than $1.2$ million images from $1,000$ categories. In the experiments, we constructed a retrieval image database including $128K$ images selected from the first $100$ largest classes from the given training data. Similarly, the $4,096$-diml deep CNN features from VGG19 $fc7$ were used for evaluation. 

\textbf{Experimental Settings:} Following the previous experimental configurations used in \cite{SDH,COSDISH}, we randomly split the CIFAR-10 dataset into a training set ($59K$ images) and a test query set ($1,000$ images), which has $100$ images per category. In SUN-397, we randomly sample $100$ images from each of the first $18$ largest scene categories to formulate $1,800$ query images. For ImageNet, we randomly select $50$ images for each category from the validation image dataset to construct the $5,000$ query image dataset. The semantic similarities on these datasets are measured whether two images share the same semantic label. For all datasets, we conduct feature normalization to make each dimension have zero-centered mean and equal variance.

\textbf{Baseline Methods and Implementation Details:} In experiments, we compare our SADIH and SADIH-L1 with 14 hashing methods including six unsupervised hashing methods, (\textit{i.e.}, ITQ, SGH, DSH \cite{jin2014density}, CBE, and OCH) and eight supervised hashing methods (\textit{i.e.}, ITQ-CCA \cite{ITQ}, KSH, LFH, COSDISH, SDH, SDHR, FSDH, FastH). All the compared algorithms were performed five times with different random initializations, and the averaged experimental results were reported. To make fair comparison, all the compared methods were reimplemented using the released source codes given by the corresponding authors. Specifically, we searched the best parameters carefully for each algorithm by five-folds cross-validation, or directly employed the default parameters suggested by the original papers. For graph based method such as KSH and OCH, using the full semantic information for training is impossible due to the heavy computation complexity, and $5,000$ samples were selected from the training data for model construction. For our SSAH, the parameter $\gamma$ was empirically set to $0.001$. For the parameters $\alpha$ and $\beta$, we should tune it by cross-validation from the candidate set $\{0.01, 0.1, 1.0,5, 10\}$. The maximum iteration number $t$ was set to $5$, which could assure the best performance. 

\begin{figure*}[!t]
\centering\resizebox{.99\textwidth}{!}{
\includegraphics[]{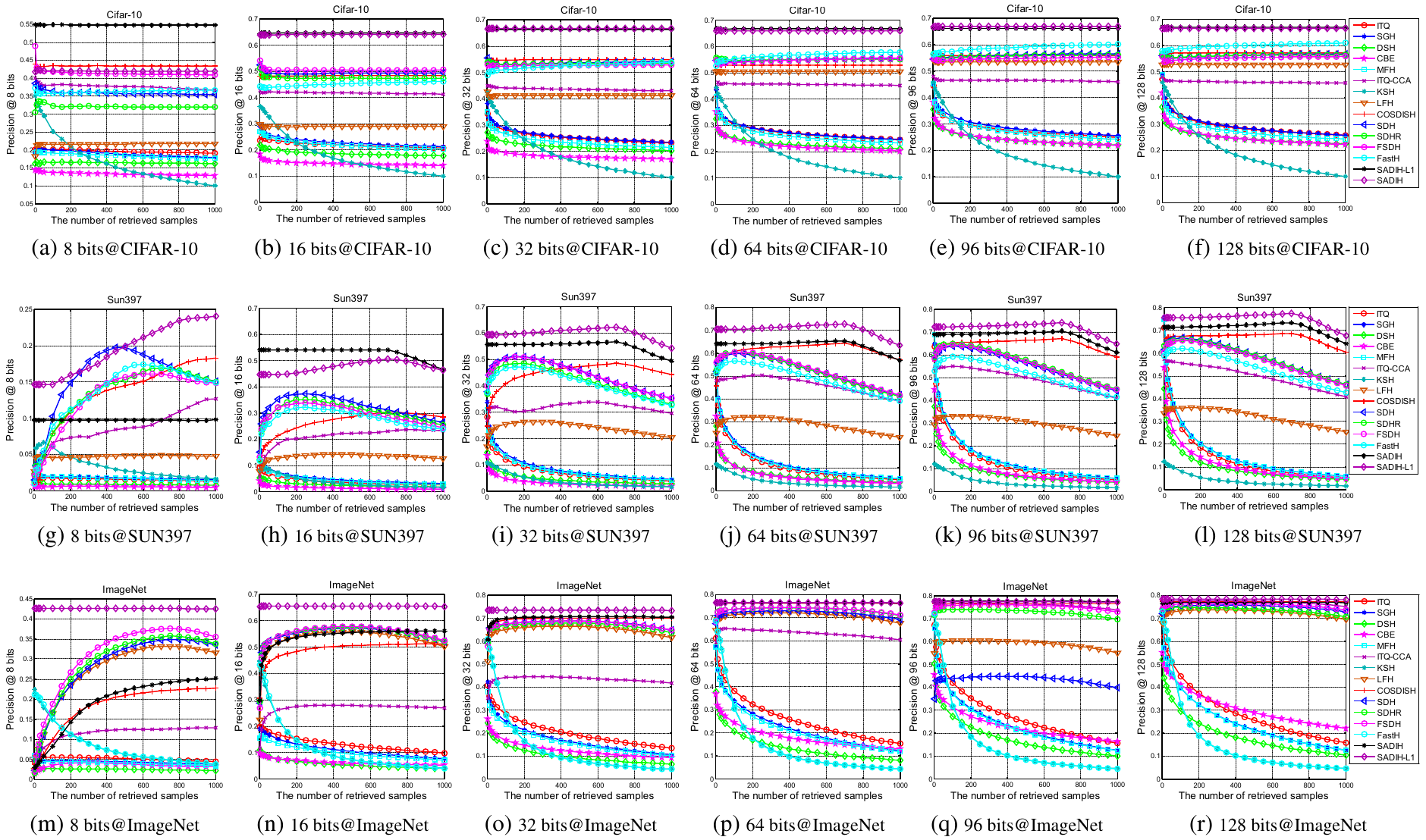}}
\caption{\small Precision curves of up to 1000 retrieved samples of different methods with different code lengths on CIFAR-10, SUN397 and ImageNet. (Better to view in color) }\label{fig_3}
\end{figure*}

\textbf{Evaluation Measures:} We adopted three frequently-used performance metrics \cite{manning2008} to evaluate different methods, \textit{i.e.} mean average precision (MAP), the precision-recall curves and mean precision rate curves of top 1000 returned samples. Moreover, we also compared the computation time to show efficiency.

\textbf{Quantitative Results:} Table \ref{Table_1} explicitly illustrates the retrieval results of different algorithms with MAP and Precision@top100 as well as the computation time on CIFAR-10, SUN397 and ImageNet datasets. 1) Generally, supervised methods can achieve higher accuracies than unsupervised methods, while OCH and ITQ can lead to competitive performances. Importantly, the unsupervised methods are unsuitable to deal with large-scale image searching with large number of classes, such as SUN397 and ImageNet, while they can handle datasets with fewer categories such as CIFAR-10. 2)  Our SADIH and SADIH-L1 in most cases can achieve the highest MAP scores with different code lengths, which demonstrate the efficacy of the proposed framework. Our SADIH and SADIH-L1 always produce superior performance to SDH, FSDH, FastH and SDHR, since our methods simultaneously consider the category information and pairwise similarities, however, other mentioned methods neglect the similarity preservation. Compared to KSH, LFH and COSDISH, our methods can make use of the $full$ similarities and discriminative category-level information for learning effective binary codes, yielding superior performance. 3) SADIH-L1 has a tendency to outperform SADIH, since $l_1$-norm regularization can generate more accurate approximation measurement. Increasing the coding lengths can improve the retrieval accuracies, but needs more training time. 4) Supervised methods seem to consume longer computation time compared to unsupervised ones. However, our SADIH-L1 is the fastest one in supervised methods, and SADIH can provide good balance between performance and time. 

We further show the precision-recall curves of the compared methods with varying code lengths in Fig. \ref{fig_2}. It can be observed that our methods are consistently better than all the competing methods, which indicates that our methods can retrieve more similar samples for a query at any fixed code length. Moreover, the precision variations \textit{w.r.t.} different number of retrieved samples are illustrated in Fig. \ref{fig_3}. We can observe that our methods are always superior to other methods, and their precisions are relatively stable with varying number of returned samples.

\section{Conclusion}\label{conc}
In this paper, we proposed a novel joint discriminative hashing framework, dubbed semantic-aware DIscrete Hashing (SADIH), which could efficiently guarantee the full semantic similarity preservation and discriminative semantic space construction. SADIH leveraged the asymmetric similarity approximation loss to preserve the full $n\times n$ similarities of the complete dataset. Meanwhile, the supervised semantic-aware autoencoder was designed to construct the discriminative semantic embedding space with full data variation preservation and good data reconstruction. The resulting problem was efficiently solved by the proposed discrete optimization algorithm. Extensive experimental results demonstrated the superiority of our methods on different large-scale datasets in terms of different evaluation protocols. 

\subsubsection{Acknowledgment} This work is partially supported by ARC FT130101530.

\bibliography{SADIHbib}
\bibliographystyle{aaai}
\smallskip

\end{document}